\documentclass[letterpaper, 10 pt, conference]{ieeeconf}
\usepackage[T1]{fontenc}
\usepackage[utf8]{inputenc}
\usepackage{float, amsfonts, graphicx, mathtools, subcaption, color, hyperref}
\usepackage{forloop}
\usepackage{epstopdf}
\graphicspath{{figures/exp/}{figures/robot_pictures/}{figures/sim/}}

\usepackage{algorithm}
\usepackage[noend]{algpseudocode}

\usepackage{amsthm}

\newtheorem{theorem}{Theorem}

\newtheorem{proposition}[theorem]{Proposition}

\DeclareMathOperator*{\argmin}{arg\,min} 
\DeclareMathOperator*{\R}{\mathbb{R}}
\DeclareMathOperator*{\+}{\!+\!}

\allowdisplaybreaks[1]
\mathtoolsset{showonlyrefs=true}

\newcommand{\sint}[2]{
    \int\displaylimits_{\mathclap{#1}}^{\mathclap{#2}}
}

\newcounter{mcspacecounter}

\usepackage{xcolor}
\usepackage{multirow}
\usepackage{color, colortbl}
\definecolor{Gray}{gray}{0.9}
\definecolor{LightGray}{gray}{0.8}
\definecolor{LightCyan}{rgb}{0.88,1,1}
\definecolor{LawnGreen}{rgb}{0.48,0.98,0}
\definecolor{mygreen}{RGB}{28,172,0} 
\definecolor{mylilas}{RGB}{170,55,241}
\definecolor{lightPaleGreen}{RGB}{152,251,152}
\definecolor{fullPaleGreen}{RGB}{173,255,47}
\definecolor{lightPaleRed}{RGB}{240,128,128}
\definecolor{fullPaleRed}{RGB}{255,99,71}

\IEEEoverridecommandlockouts 

\mathtoolsset{showonlyrefs=true}

\title{A Parametric MPC Approach to Balancing the Cost of Abstraction for Differential-Drive Mobile Robots}

\begin{document}

\author{Paul~Glotfelter and Magnus~Egerstedt
\thanks{This research was sponsored by Grants No. 1531195 from the U.S. National Science Foundation.}
\thanks{
The authors are with the Institute for Robotics and Intelligent Machines, Georgia Institute of Technology, Atlanta, GA 30332, USA, \{paul.glotfelter,magnus\}@gatech.edu.}%
}

\IEEEtriggeratref{1}

\maketitle
\thispagestyle{empty}
\pagestyle{empty}

\begin{abstract}


When designing control strategies for differential-drive mobile robots, one standard tool is the consideration of a point at a fixed distance along a line orthogonal to the wheel axis instead of the full pose of the vehicle.  This abstraction supports replacing the non-holonomic, three-state unicycle model with a much simpler two-state single-integrator model (i.e., a velocity-controlled point).  Yet this transformation comes at a performance cost, through the robot's precision and maneuverability.  This work contains derivations for expressions of these precision and maneuverability costs in terms of the transformation's parameters.  Furthermore, these costs show that only selecting the parameter once over the course of an application may cause an undue loss of precision.  Model Predictive Control (MPC) represents one such method to ameliorate this condition.  However, MPC typically realizes a control signal, rather than a parameter, so this work also proposes a Parametric Model Predictive Control (PMPC) method for parameter and sampling horizon optimization.  Experimental results are presented that demonstrate the effects of the parameterization on the deployment of algorithms developed for the single-integrator model on actual differential-drive mobile robots.

\end{abstract}

\section{Introduction}


Models are always abstractions in that they capture some pertinent aspects of the system under consideration whereas they neglect others. But models only have value inasmuch as they allow for valid predictions or as generators of design strategies. For example, in a significant portion of the many recent, multi-agent robotics algorithms for achieving coordinated objectives, single-integrator models are employed (e.g., \cite{Ji2006,Y2013,Hong2006,Ni2010}). Arguably, such simple models have enabled complex control strategies to be developed, yet, at the end of the day, they have to be deployed on actual physical robots.  This paper formally investigates how to strike a balance between performance and maneuverability when mapping single-integrator controllers onto differential-drive mobile robots. 

Due to the single-integrator model's prevalence as a design tool, a number of methods have been developed for mapping from single-integrator models to more complex, non-holonomic models.  For example, the authors of \cite{Cortes2002} achieve a map from single integrator to unicycle by leveraging a control structure introduced in \cite{Astolfi1999}.  However, this map does not come with formal guarantees about the degree to which the unicycle system approximates the single-integrator system.  One effective solution to this problem is to utilize a so-called Near-Identity Diffeomorphism (NID) between single-integrator and unicycle systems, as in \cite{Olfati-Saber2002}, \cite{Olfati-Saber}, where the basic idea is to perturb the original system ever-so-slightly (the near-identity part) and then show that there exists a diffeomorphism between a lower-dimensional version of the perturbed system's dynamics and the single-integrator dynamics. As the size of the perturbation is given as a design parameter, a bound on how far the original system may deviate from the single-integrator system follows automatically.

A concept similar to NIDs from single-integrator to unicycle dynamics appears in the literature in different formats.  For example, \cite{Ogren2001} utilizes this technique from a kinematics viewpoint to stabilize a differential-drive-like system.  This "look-ahead" technique also arises in feedback linearization methods as a mathematical tool to ensure that the differential-drive system is feedback linearizable (e.g., \cite{Oriolo2002, Yang2005}).


This paper utilizes the ideas in \cite{Olfati-Saber2002}, \cite{Olfati-Saber} to show that the NID incurs an abstraction cost, in terms of precision and maneuverability, that is based on the physical geometry of the differential-drive robots; in particular, the precision cost focuses on increasing the degree to which the single-integrator system matches the unicycle-modeled system, and the maneuverability cost utilizes physical properties of the differential-drive systems to limit the maneuverability requirements imposed by the transformation. By striking a balance between these two costs, a one-parameter family of abstractions arises.  However, the maneuverability cost shows that only selecting the parameter once over the course of an experiment may cause a loss of precision.

A potential solution to this issue is to repeatedly optimize the parameter based on the system's model and a suitable cost metric.  Model Predictive Control (MPC) represents one such method.  In particular, MPC approaches solve an optimal control problem over a time interval, utilize a portion of the controller, and re-solve the problem over the next time interval, effectively producing a state- and time-based controller.  The authors of \cite{Droge2011, Droge2011-2} produce such a Parametric Model Predictive Control (PMPC) formulation.  However, this formulation does not permit the cost metric to influence the time interval, which has practical performance implications.  Using the formulated precision and maneuverability costs, this work formulates an appropriate PMPC cost metric and extends the work in \cite{Droge2011, Droge2011-2} to integrate a sampling horizon cost directly into the PMPC program.


This paper is organized as follows:  Sec.~\ref{sec:system-of-interest} presents the system of interest and introduces the inherent trade-off contained in the NID.  Sec.~\ref{sec:a-parametric-mpc-formulation} discusses the PMPC formulation.   Sec.~\ref{sec:main-results} formulates the cost functions that allow a balanced selection of the NID's parameters, with respect to the generated cost functions.  To demonstrate and verify the main results of this work, Sec.~\ref{sec:numerical-results} shows data from simulations and physical experiments, with Sec.~\ref{sec:conclusion} concluding the paper.

\section{From Unicycles to Single Integrators}
\label{sec:system-of-interest}

This article uses the following mathematical notation. The expression $\|\cdot\|$ is the usual Euclidean norm.  The symbol $\partial_{x} f(x)$ represents the partial derivative of the function $f : \R^{n} \to \R^{m}$ with respect to the variable $x$, assuming the convention that $\partial_{x} f(x) \in \R^{m \times n}$.  The symbol $\R_{\geq 0}$ refers to the real numbers that are greater than or equal to zero.

As the focus of the paper is effective abstractions for controlling differential-drive robots, this section establishes the Near-Identity Diffeomorphism (NID) that provides a relationship between single-integrator and unicycle models.  That is,  systems whose pose is given by planar positions $\bar{x}=\left[x_1~ x_2 \right]^T$ and orientations $\theta$, with the full state given by $x=\left[ \bar{x}^T~ \theta \right]^T=\left[ x_1~ x_2~ \theta \right]^T$. The associated unicycle dynamics are given by (dropping the dependence on time $t$)
\begin{equation}
    \dot{x} = 
    \begin{bmatrix}
        R(\theta)e_{1} & {0}\\ 
        \mathbf{0} & 1
    \end{bmatrix}
    \begin{bmatrix}
        v \\ 
        \omega
    \end{bmatrix},
    \label{eq:unicycle-model}
\end{equation}
where the control inputs $v, \omega \in \mathbb{R}$ are the linear and rotational velocities, respectively, $\mathbf{0}$ is a zero-vector of the appropriate dimension, and 
\begin{equation}
    e_{1} = 
    \begin{bmatrix}
        1 & 0
    \end{bmatrix}^{T}, ~
    R(\theta) = 
    \begin{bmatrix}
        \cos(\theta) & -\sin(\theta) \\ 
        \sin(\theta) & \cos(\theta)
    \end{bmatrix}.
\end{equation}

Letting 
\begin{equation}
    u_{x} = 
    \begin{bmatrix}
        v & \omega 
    \end{bmatrix}^{T}
\end{equation}
be the collective control input to the unicycle-modeled agent, the objective becomes to turn this model into a single-integrator model. To this end, we here recall the developments in \cite{Olfati-Saber2002}.  Let $x_{si} \in \mathbb{R}^2$ be given by
\begin{equation}
    x_{si} = \Phi(x, l) = \bar{x} + l R(\theta)e_{1},
    \label{eq:nid}
\end{equation}
where $l \in (0, \infty)$  is a constant.  The map $\Phi(x, l)$ is, in fact, the NID, as defined in \cite{Olfati-Saber2002}.  Geometrically, the point $x_{si}$ is simply given by a point at a distance $l$ directly in front of the unicycle with pose  $x$.

Now, assume that the dynamics of $x_{si}$ are given by a controller
\begin{equation}
    \dot{x}_{si} = u_{si} ,
\end{equation}
where $u_{si} \in \R^{2}$ is continuously differentiable, and compare this system to the time-derivative of \eqref{eq:nid}, which yields 
\begin{equation}
    \dot{x}_{si} = u_{si} =
    \begin{bmatrix}
        \cos(\theta) & -l\sin(\theta) \\ 
        \sin(\theta) & l\cos(\theta)
    \end{bmatrix}
    u_{x}
    = R_{l}(\theta)u_{x} .
    \label{eq:diffeomorphism}
\end{equation}
Note that the NID maps from three degrees of freedom to two degrees of freedom.  As a consequence, the resulting unicycle controller cannot explicitly affect the orientation $\theta$ of the unicycle model.

By \cite{Olfati-Saber2002}, $R_{l}(\theta)$ is invertible, yielding a relationship between $u_{si}$ and $u_{x}$.  Consequently, \eqref{eq:diffeomorphism} allows the transformation of linear, single-integrator algorithms into algorithms in terms of the non-linear, unicycle dynamics.  Note that in this paper, which is different from \cite{Olfati-Saber2002}, we let $\dot{l} = 0$ over the PMPC time intervals (i.e., $l$ is a constant value).

The unicycle model in \eqref{eq:unicycle-model} is not directly realizable on a differential-drive mobile robot. However, the relationship between the control inputs to the unicycle model and the differential-drive model is given by
\begin{equation}
    v = \dfrac{r_{w}}{2}(\omega_{r} + \omega_{l}), ~ \omega = \dfrac{r_{w}}{l_{w}}(\omega_{r} - \omega_{l}),
    \label{eq:differential-drive}
\end{equation}
where $\omega_{r}$ and $\omega_{l}$ are the right and left wheel velocities, respectively.  The wheel radius $r_{w}$ and base length $l_{w}$ encode the geometric properties of the robot.

In the discussion above, the parameter $l$ (i.e., the distance off the wheel axis to the new point) is not canonical.  Moreover, it plays an important role since
\begin{equation} 
    \|\bar{x} - x_{si}\| = l.
    \label{eq:geo-rel}
\end{equation}
The above equation seems to indicate that one should simply choose $l \in (0,~ \infty)$ to be as small as possible.  However, the following sections show that small values of $l$ induce high maneuverability costs.  

In order to strike a balance between precision and maneuverability, we will, 
for the remainder of this paper, assume that the control input to the unicycle model is given by 
\begin{equation}
    u_{x} = R_{l}(\theta)^{-1}u_{si},
    \label{eq:si-to-uni}
\end{equation}
where $u_{si}$ is the control input supplied by a single-integrator algorithm.  Sec.~\ref{sec:main-results} contains the further investigation of the effects of the parameter $l$ on the precision and maneuverability implications of the transformation in \eqref{eq:nid}.


\section{A Parametric MPC Formulation}
\label{sec:a-parametric-mpc-formulation}

Having introduced the system of interest, this section contains a derivation of a Parametric Model Predictive Control (PMPC) method with a variable sampling interval for general, nonlinear systems.  Later, Sec.~\ref{sec:numerical-results} utilizes a specific case of these results.  In general, MPC methods solve an optimal control problem over a time interval and use only a portion of the obtained controller (for a small amount of time) before resolving the problem, producing a time- and state-based controller.  In this case, PMPC optimizes the parameters of a system.  That is, this method finds the optimal, constant parameters of a system, rather than a time-varying control input, over a time interval.  For clarity, this section specifies dependencies on time $t$.  Let 
\begin{equation}
    \dot{x}(t) = f(x(t), p, t),~ x_{t_{0}} = x(t_{0}) ,
\end{equation}
where $x(t) \in \R^{n}$, $p \in \R^{m}$, and $f(\cdot)$ is continuously differentiable in $x$, measurable in $t$.  The program
\begin{align}
    & \argmin_{p \in \R^{m}, \Delta t \in \R_{\geq 0}} J(p, \Delta t) = \sint{t_{0}}{t_{0} + \Delta t} L(x(s), p, s) ds + C(\Delta t) \\
    & \text{s.t. } \dot{x}(t) = f(x(t), p, t) \\
    & \quad~ x(t_{0}) = x_{t_{0}} ,
\end{align}
expresses the PMPC problem of interest, where $L(\cdot)$ is continuously differentiable in $x$ and $p$.  Note that, in this case, both $\Delta t$ and $p$ are decision variables determined by the PMPC program.

\subsection{Optimality Conditions}

This section contains the derivation of the necessary, first-order optimality conditions for the PMPC formulation, realizing gradients for the proposed cost.  In particular, the derivation proceeds by calculus of variations.
\begin{proposition}
    \label{prop:mpc-like}
    The augmented cost derivatives $\partial_{p}\tilde{J}(p, \Delta t)$, $\partial_{\Delta t}\tilde{J}(p, \Delta t)$ are
    \begin{align*}
        &\partial_{p}\tilde{J}(p, \Delta t) = \sint{t_{0}}{t_{0} + \Delta t} \partial_{p}L(x(s), p, s) + \lambda(s)^{T}\partial_{p}f(x(s), p, s) ds \\
        &\partial_{\Delta t}\tilde{J}(p, \Delta t) = L(x(t_{0} + \Delta t), p, t_{0} + \Delta t) + \partial_{\Delta t}C(\Delta t) ,
    \end{align*}
    where the augmented cost $\tilde{J}(p, \Delta t)$ (i.e., $J(p, \Delta t)$ augmented with the dynamics constraint) is given by 
    \begin{equation}
        \begin{split}
        & \tilde{J}(p, \Delta t) = \\
        & \sint{t_{0}}{t_{0} + \Delta t} L(x(s), p, s) \+ \lambda(s)^{T}(f(x(s), p, s) \!-\! \dot{x}(s)) ds \+ C(\Delta t) .
        \end{split}
    \end{equation}
\end{proposition}
\begin{proof}
    The proof proceeds by calculus of variations.  Perturb $p$ and $\Delta t$ as $p \mapsto p + \epsilon \gamma$ and $\Delta t \mapsto \Delta t + \epsilon \tau$, where $\gamma \in \R^{m}$, $\tau \in \R$.  The perturbed augmented cost is
    \begin{equation}
        \begin{split}
            & \tilde{J}(p \+ \epsilon \gamma, \Delta t \+ \epsilon \tau) = \\
            & \sint{t_{0}}{t_{0} + \Delta t \+ \epsilon \tau} L(x(s) \+ \epsilon \eta(s), p \+ \epsilon \gamma, s) \\
            &\qquad \lambda(s)^{T}(f(x(s) \+ \epsilon \eta(s), p \+ \epsilon \gamma, s) \!-\! \dot{x}(s) \!-\! \epsilon \dot{\eta}(s))
            ds \+ \\
            & C(\Delta t \+ \epsilon \tau) \+ o(\epsilon) .
        \end{split}
    \end{equation}
    Performing a Taylor expansion yields that 
    \begin{equation}
        \begin{split}
             & \tilde{J}(p \+ \epsilon \gamma, \Delta t \+ \epsilon \tau) = \\
             & \sint{t_{0}}{t_{0} \+ \Delta t \+ \epsilon \tau} L(x(s), p, s) \+ \epsilon \partial_{x}L(x(s), p, s)\eta(s) \+ \epsilon\partial_{p}L(x(s), p, s)\gamma \\
             & \quad \+ \lambda^{T}(f(x(s), p, s) \+ \epsilon \partial_{x}f(x(s), p, s)\eta(s) \\
             & \quad\qquad \+ \epsilon\partial_{p}f(x(s), p, s)\gamma \!-\! \dot{x}(s) \!-\! \epsilon \dot{\eta}(s)) ds \\ 
             & \+ C(\Delta t) + \epsilon \partial_{\Delta t}C(\Delta t) \tau \+ o(\epsilon) .
        \end{split}
    \end{equation}
    
    The proof now proceeds with multiple steps.  First, the application of integration by parts to the quantity $\lambda(t)^{T}\dot{\eta}(t)$.  Second, the subtraction of the costs $\tilde{J}(p + \epsilon \gamma, \Delta t + \epsilon \tau) - \tilde{J}(p, \Delta t)$.  Note that, to subtract the costs properly, the integral in $\tilde{J}(p + \epsilon \gamma, t + \epsilon \tau)$ must be broken up into two intervals: $[t, t + \Delta t]$ and $[t + \Delta t, t + \Delta t + \tau]$.  Furthermore, the costate assumes the usual definition: $\dot{\lambda}(t) = -\partial_{x}L(x(t), p, t)^{T} - \partial_{x}f(x(t), p, t)^{T}\lambda(t)$ with the boundary condition $\lambda(t_{0} + \Delta t) = 0$.  Applying the mean value theorem and taking the limit as $\epsilon \to 0$ shows that
    \begin{equation}
        \label{eq:cost-limit}
        \begin{split}
           & \lim_{\epsilon \to 0} \dfrac{\tilde{J}(p \+ \epsilon \gamma, \Delta t \+ \epsilon \tau) - \tilde{J}(p, \Delta t)}{\epsilon} = \\
            & \left[~ \sint{t_{0}}{t_{0} \+ \Delta t} \partial_{p}L(x(s), p, s) \+ \lambda(s)^{T} \partial_{p}f(x(s), p, s) ds \right] \gamma \\
            & \+ \left[ \partial_{\Delta t}C(\Delta t) \+ L(x(t_{0} \+ \Delta t), p, t_{0} \+ \Delta t) \right] \tau,
        \end{split}
    \end{equation}
    which is linear in $\tau$ and $\gamma$, and provides the final expressions 
    \begin{align}
        & \partial_{p}\tilde{J}(p, \Delta t) = \sint{t_{0}}{t_{0} \+ \Delta t} \partial_{p}L(x(s), p, s) \+ \lambda(t)^{T}  \partial_{p}f(x(s), p, s) ds \\
        & \partial_{\Delta t}\tilde{J}(p, \Delta t) = \partial_{\Delta t}C(\Delta t) \+ L(x(t_{0} + \Delta t), p, t_{0} \+ \Delta t) ,
    \end{align}
    completing the proof.
\end{proof}

Interestingly, both of the usual conditions for free parameters and final time still hold, and the first-order, necessary optimality conditions for candidate solutions $p^{*}$ and $\Delta t^{*}$ are that 
\begin{equation}
    \partial_{p}\tilde{J}(p^{*}, \Delta t^{*}) = 0,~ \partial_{\Delta t}\tilde{J}(p^{*}, \Delta t^{*}) = 0 .
\end{equation}

Furthermore, this formulation becomes amenable to solution by numerical methods for the optimal parameters $p^{*}$ and $\Delta t^{*}$.  In such cases, the expression for $\partial_{p}\tilde{J}(p, \Delta t)$ can also be expressed as a costate-like variable $\xi : [t_{0}, t_{0} + \Delta t] \to \R^{m}$ with dynamics
\begin{align}
    & \dot{\xi}(t) = -\partial_{p}L(x(t), p, t)^{T} - \partial_{p}f(x(t), p, t)^{T}\lambda(t) \\
    &\xi(t_{0} + \Delta t) = 0 ,
\end{align}
where $\xi(\cdot)$ is defined as 
\begin{equation}
    \xi(t) = \sint{t}{t_{0} \+ \Delta t} \partial_{p}L(x(s), p, s)^{T} \+ \partial_{p}f(x(s), p, s)^{T}\lambda(s) ds .
\end{equation}
In this case, the necessary optimality condition is that 
\begin{equation}
    \xi(t_{0}) = 0 .
\end{equation}
 
\subsection{Numerical Methods}

The above expressions allow for applications of typical gradient descent methods.  Many such methods could apply, and this article presents one simple method in Alg.~\ref{alg:mpc}.  Note that this algorithm procures the decision variables over one sampling interval $[t_{0}, t_{0} + \Delta t]$.  In practice, one typically applies this algorithm repeatedly.
\begin{algorithm}[ht]
    \caption{Gradient Descent Algorithm for PMPC}
    \label{alg:mpc}
    \begin{algorithmic}[1]
        \State $k \gets 0$
        \State $p_{k} \gets \text{initial guess}$
        \State $\Delta t_{k} \gets \text{initial guess}$
        \While {$\left\| \partial_{p}\tilde{J}(p_{k}, \Delta t) \right\| + \left\| \partial_{\Delta t}\tilde{J}(p_{k}, \Delta t) \right\| > \epsilon$}
            \State Solve forward for $x(\cdot)$ from $x_{t}$ using $p_{k}$ and $\Delta t_{k}$
            \State Solve backward for $\lambda(\cdot)$, $\xi(\cdot)$ using $x(\cdot)$
            \State Compute gradients $\partial_{p}\tilde{J}(p_{k}, \Delta t)$ and $\partial_{\Delta t}\tilde{J}(p_{k}, \Delta t)$
            \State $p_{k+1} \gets p_{k} - \gamma_{1} \partial_{p}\tilde{J}(p_{k}, \Delta t)^{T}$
            \State $\Delta t_{k+1} \gets \Delta t_{k} - \gamma_{2} \partial_{\Delta t}\tilde{J}(p_{k}, \Delta t)$
            \State $k \gets k + 1$
        \EndWhile
    \end{algorithmic}
\end{algorithm}
For example, the experiments in Sec.~\ref{subsec:experimental-comparison} consecutively apply this algorithm to solve the PMPC problem.


\section{Precision vs. Maneuverability} 
\label{sec:main-results}
As already noted in Sec.~\ref{sec:system-of-interest}, the parameter $l$ is a design parameter.  This section discusses the importance and effects of selecting $l$ and proposes precision and maneuverability costs that elucidate the selection of this parameter and its impact on the differential-drive system.  These derivations influence the PMPC cost metric in Sec.~\ref{sec:numerical-results} and, for comparison, an optimal, static parameterization. 

\subsection{Precision Cost}
\label{subsec:precision-cost}

Seeking to select $l$, we initially present a cost that incorporates the degree to which the transformed system in \eqref{eq:nid} represents the original system $x_{si}$ over an arbitrary time duration $T \geq 0$.  As such, we model the precision cost by the averaged tracking error
\begin{equation}
    D_{1}(\bar{x}, x_{si}) = \dfrac{1}{T} \int_{0}^{T}\|\bar{x} - x_{si}\|~dt.
    \label{eq:precision-cost}
\end{equation}
It immediately follows from \eqref{eq:geo-rel} that $D_{1}(\bar{x}, p)$ can be directly written as a function of $l$, given by
    \begin{align}
        D_{1}(\bar{x}, x_{si}) &= \dfrac{1}{T} \int_{0}^{T}\|\bar{x} - x_{si}\|~dt \\
        &= \dfrac{1}{T} \int_{0}^{T} l ~dt =l.
        \label{eq:precision-cost-lambda}
    \end{align}
This immediate result states that the smaller $l$ is, the better the unicycle model tracks the single-integrator model.

\subsection{Maneuverability Cost} 
\label{subsec:man-cost}
In this section, we derive a geometrically-influenced maneuverability cost that models the degree to which the selection of $l$ influences the maneuverability requirements of the unicycle-modeled system, with respect to the map defined in \eqref{eq:nid}.  That is, we wish to elucidate how the parameter $l$ affects the expressions for the differential-drive agent's forward velocity, wheel difference, and exerted control effort.

To this end, we utilize the differential-drive model in \eqref{eq:differential-drive}.  Initially, note that the magnitude of the wheel-velocity difference
$
    \left| \omega_{r} - \omega_{l} \right|
$
represents a measure of the complexity of a maneuver that the differential-drive system performs.  Using this definition as guidance, we state the following proposition.
\begin{proposition}
    \label{prop:wheel-velocity-difference-upper-bound}
    Given that the control-input magnitude $\|u_{si}\|$ is upper-bounded by $\bar{v}$, the magnitude of the wheel-velocity difference, $\left| \omega_{r} - \omega_{l} \right|$, is upper bounded by 
    \begin{equation}
        \left| \omega_{r} - \omega_{l} \right| \leq \dfrac{l_{w}\bar{v}}{r_{w}\lambda}.
    \end{equation}
\end{proposition}
\begin{proof}
    Let 
    \begin{equation}
        e_{2} = 
        \begin{bmatrix}
            0 & 1
        \end{bmatrix}^{T}
        ,~ T^{-1} = 
        \begin{bmatrix}
            1 & 0 \\
            0 & \dfrac{1}{l}
        \end{bmatrix}
        ,~ \bar{u}_{si} =
        \begin{bmatrix}
            \|u_{si}\| & 0
        \end{bmatrix}^{T}
    \end{equation}
    and let $\theta_{si}$ be the angle of the vector $u_{si}$.  From \eqref{eq:diffeomorphism},\eqref{eq:differential-drive} we can retrieve the magnitude of the difference in angular velocities, $\left| \omega_{r} - \omega_{l} \right|$, as
    \begin{align}
        &\left| \omega_{r} - \omega_{l} \right| = \left| \dfrac{l_{w}}{r_{w}}\omega \right| \\
        &= \left| \dfrac{l}{r} e_{2}^{T}R_{l}(\theta)^{-1}u_{si} \right| \\
        &= 
            \left|
            \dfrac{l_{w}}{r_{w}}
            e_{2}^{T}
            T^{-1}
            R(-\theta) 
            R(\theta_{si}) 
            \bar{u}_{si} 
            \right| \\
        &= 
            \left|
            \dfrac{l_{w}}{r_{w}}
            e_{2}^{T}
            T^{-1}
            R(\theta_{si}-\theta)
            \bar{u}_{si} 
            \right| \\
        &= 
            \left|
            \dfrac{l_{w}}{r_{w}}
            \begin{bmatrix}
                0 & \dfrac{1}{l}
            \end{bmatrix} 
            \begin{bmatrix}
                \cos(\theta_{si} - \theta) & -\sin(\theta_{si} - \theta) \\ 
                \sin(\theta_{si} - \theta) & \cos(\theta_{si} - \theta)
            \end{bmatrix}
            \begin{bmatrix}
                \|u_{si}\| \\ 
                0
            \end{bmatrix} 
            \right| \\
        &= 
            \left|
            \dfrac{l_{w}}{r_{w}}
            \begin{bmatrix}
                \dfrac{\sin(\theta_{si} - \theta)}{l} &  \dfrac{\cos(\theta_{si} - \theta)}{l}
            \end{bmatrix}
            \begin{bmatrix}
                \|u_{si}\| \\ 
                0
            \end{bmatrix} 
            \right|\\
        &= \left| \dfrac{l_{w}\|u_{si}\|}{r_{w}l}\sin(\theta_{si} - \theta) \right| \\ 
        &\leq \dfrac{l_{w}\bar{v}}{r_{w} l}.
    \end{align}
\end{proof}
Thus, Prop.~\ref{prop:wheel-velocity-difference-upper-bound} yields an upper bound on the magnitude of the wheel-velocity difference
\begin{equation} 
     \left| \omega_{r} - \omega_{l} \right| \leq \dfrac{l_{w}\bar{v}}{r_{w} l}.
     \label{eq:wheel-difference-upper-bound}
\end{equation}
Prop.~\ref{prop:linear-velocity} shows a similar result for the forward velocity of the differential-drive agent.

\begin{proposition}
    Given that the control-input magnitude $\|u_{si}\|$ is upper-bounded by $\bar{v}$, the magnitude of the forward velocity, $\left| \omega_{r} + \omega_{l} \right|$, is upper bounded by 
    \begin{equation}
        \left| \omega_{r} + \omega_{l} \right| \leq \dfrac{2\bar{v}}{r_{w}}.
        \label{eq:forward-velocity-upper-bound}
    \end{equation} 
    \label{prop:linear-velocity}
\end{proposition}
\begin{proof}
    \allowdisplaybreaks
    Let 
    \begin{equation}
        e_{1} = 
        \begin{bmatrix}
            1 & 0
        \end{bmatrix}^{T}
    \end{equation} 
    and 
    \begin{equation} 
        T^{-1}, \bar{u}_{si}, \theta_{si}
    \end{equation}
    be defined as in the proof of Prop.~\ref{prop:wheel-velocity-difference-upper-bound}.  Then, we have, through \eqref{eq:differential-drive}, that
    \begin{align}
        \left| \omega_{r} + \omega_{l} \right| &= \left| \dfrac{2}{r_{w}}v \right| \\ 
        &= \left| \dfrac{2}{r_{w}}e_{1}^{T}R_{l}(\theta)^{-1}u_{si} \right| \\
        &= \left| \dfrac{2}{r_{w}}e_{1}^{T}T^{-1}R(-\theta)R(\theta_{si})\bar{u}_{si} \right| \\ 
        &= \left| \dfrac{2}{r_{w}}\left[ 1 ~ 0 \right] R(\theta_{si} - \theta)\bar{u}_{si} \right| \\ 
        &= \left| \dfrac{2}{r_{w}}\left[ \cos(\theta_{si} - \theta) ~ -\sin(\theta_{si} - \theta) \right] 
        \begin{bmatrix}
            \|u_{si}\| \\ 
            0
        \end{bmatrix} 
        \right| \\ 
        &= \left| \dfrac{2}{r_{w}} \|u_{si}\|\cos(\theta_{si} - \theta) \right| \\
        &\leq \dfrac{2\bar{v}}{r_{w}}.
    \end{align}
\end{proof}
So Prop.~\ref{prop:linear-velocity} reveals that the forward velocity of the differential-drive system remains independent of the selection of the parameter $l$.  

To elucidate an appropriate maneuverability cost in terms of $l$, define the average control effort exerted by the differential-drive system over an arbitrary time duration $T >= 0$ as 
\begin{equation}
   \dfrac{1}{T}\int_{0}^{T} \left| \omega_{r} - \omega_{l} \right| + \left| \omega_{r} + \omega_{l} \right| dt .
\end{equation}
Directly applying Props.~\ref{prop:wheel-velocity-difference-upper-bound},\ref{prop:linear-velocity} reveals that the above expression is bounded above by
\begin{equation}
    \dfrac{l_{w}\bar{v}}{r_{w}l} + \dfrac{2\bar{v}}{r} .
    \label{eq:control-energy-bound}
\end{equation}

The expression in \eqref{eq:control-energy-bound} demonstrates an interesting quality of the system.  As $l$ grows large, the forward velocity dominates the control effort exerted by the differential-drive system.  However, if $l$ becomes small, then the choice of $l$ affects the potentially exerted control effort.

Thus, \eqref{eq:control-energy-bound} reveals how $l$ affects the maneuverability requirements imposed by the abstraction.  The fact that we always pay the forward-velocity price, regardless of the selection of $l$, naturally excludes the forward velocity from the soon-to-be-formulated cost, because any selection of $l$ results in the same cost bound; but the choice of $l$ directly affects the cost associated with the wheel difference.  Accordingly, the wheel difference must play a role in the final PMPC cost metric.  With this conclusion in mind, we define the static maneuverability cost as
\begin{equation}
    D_{2}(l) = \dfrac{l_{w}\bar{v}}{r_{w} l} ,
    \label{eq:maneuverability-cost}
\end{equation}
which the static parameterization in the following section utilizes.

\subsection{An Optimal, One-Time Selection}
\label{subsec:an-optimal-one-time-selection}

Sec.~\ref{sec:numerical-results} utilizes the results in Sec.~\ref{sec:main-results} to formulate an appropriate cost metric for a PMPC program.  To have a baseline comparison, this section formulates an optimal, one-time selection for the parameter $l$.  That is, the selection occurs once over the experiment's duration.  This selection should strike a balance between precision and maneuverability.  Eqns.~\eqref{eq:precision-cost-lambda} and \eqref{eq:maneuverability-cost} represent each of these facets, respectively, and introduce an inherent trade-off in selecting $l$.  Making $l$ smaller directly reduces the cost in \eqref{eq:precision-cost-lambda}.  However, consider the relationship in \eqref{eq:maneuverability-cost}; as $l$ decreases, the differential-drive system accumulates a higher maneuverability cost.  

As such, the convex combination of \eqref{eq:precision-cost} and \eqref{eq:maneuverability-cost} yields a precision and maneuverability cost in terms of $l$ as 
\begin{align}
    D(l) &= \alpha D_{1}(l) + (1 - \alpha)D_{2}(l) \nonumber \\
    &= \alpha l + (1 - \alpha)\dfrac{l_{w}\bar{v}}{r_{w}l},
    \label{eq:total-cost}
\end{align}
where $\alpha \in (0, 1)$.  Now, we seek the optimal $l$ such that \eqref{eq:total-cost} is minimized.  That is, 

\begin{equation}
    l^{*} = \argmin_{l}~D(l).
    \label{eq:minimization}
\end{equation}
\eqref{eq:minimization} leads to Prop.~\ref{prop:cost-minimization}.
\begin{proposition}
    \label{prop:cost-minimization}
    The optimal $l^{*}$ is given by 
    \begin{equation}
        l^{*} = \sqrt{\dfrac{1 - \alpha}{\alpha}\dfrac{l_{w}\bar{v}}{r_{w}}}.
        \label{eq:optimal-lambda}
    \end{equation}
\end{proposition}
\begin{proof}
    We have that
    \begin{align}
        \dfrac{\partial}{\partial l} D(l) &= \alpha - (1 - \alpha)\dfrac{l_{w}\bar{v}}{r_{w}l^{2}} \\
        &= \alpha - (1 - \alpha)\dfrac{l_{w}\bar{v}}{r_{w}l^{2}}.
    \end{align}  
    Setting this equation equal to zero directly yields the minimizer
    \begin{equation}
        \label{eq:optimal_static_l}
        l^{*} = \sqrt{\dfrac{1 - \alpha}{\alpha}\dfrac{l_{w}\bar{v}}{r_{w}}}.
    \end{equation}
    \end{proof}
    
Note that the above result utilizes \eqref{eq:maneuverability-cost}, which is an upper bound on the wheel velocity difference.  Thus, the PMPC method should outperform this static selection, a suspicion that Sec.~\ref{sec:numerical-results} investigates.

\subsection{PMPC Cost}

This section formulates a PMPC cost based on the analysis in Sec.~\ref{sec:main-results}.  To increase precision, the parameter $l$ must be minimized.  However, \eqref{eq:maneuverability-cost} in Sec.~\ref{subsec:man-cost} indicates that the wheel velocity difference must be managed.  Thus, precision and maneuverability are balanced with the cost
\begin{align}
    &L(x, l, t) = (1 - \beta)(\omega_{r} - \omega_{l})^{2} + \beta l^{2} \\
    &=  (1 - \beta)((l_{w}/r_{w})e_{2}R_{l}(\theta)^{-1}(u_{si}))^{2} + \beta l^{2} ,
\end{align}
where $\beta \in (0, 1)$.

With this cost metric, the PMPC program becomes
\begin{align}
    & \argmin_{l \in \R, \Delta t \in \R_{\geq 0}} \sint{t_{0}}{t_{0} \+ \Delta t} (1 - \beta)((l_{w}/r_{w})e_{2}R_{l}(\theta)^{-1}(u_{si}))^{2} + \beta l^{2} \\
    & \qquad\qquad\quad + C(\Delta t) \\
    \label{eq:pmpc_cost}
    & \text{s.t. } \dot{x} = 
    \begin{bmatrix}
        R(\theta)e_{1} & 0 \\
        \mathbf{0} & 1
    \end{bmatrix} 
    R_{l}(\theta)^{-1}u_{si} \\
    & \quad~ x(t_{0}) = x_{t_{0}} ,
\end{align}
Note that the sampling cost $C(\Delta t)$ and single-integrator control input $u_{si}$ have yet to be specified.

\section{Numerical Results} 
\label{sec:numerical-results}

\begin{figure}[tbp]
    \centering
    \includegraphics[width=0.48\textwidth]{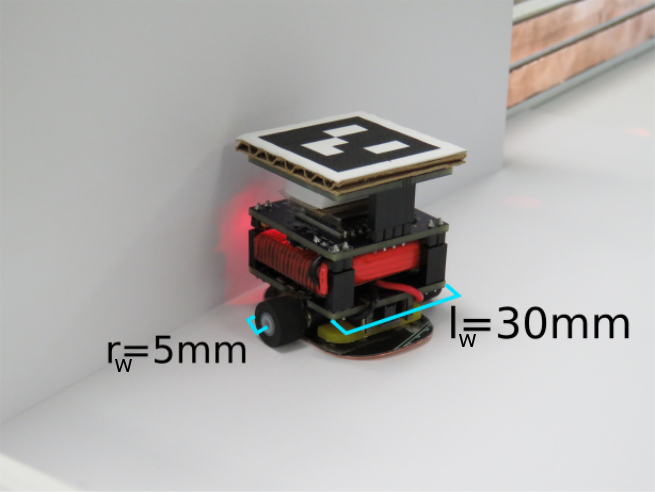}
    \caption{The GRITSbot, which is a small, differential-drive mobile robot used in the Robotarium.  This figure displays the base length and wheel radius of the GRITSbots.}
    \label{fig:gritsbot}
\end{figure}

To demonstrate the findings in Sec.~\ref{sec:main-results}, we conduct two separate tests: in simulation and on real hardware.  The simulation portion shows the effects of a one-time parameter selection on the angular velocity versus the PMPC method.  The experimental section contains the same implementation on a real, physical system: the Robotarium (\url{www.robotarium.org}).  In particular, the experimental results highlight the practical differences between using a PMPC approach and a one-time selection.

\subsection{Experiment Setup}
\label{subsec:experiment_setup}

\begin{figure}[bp]
    \centering
    \includegraphics[width=0.235\textwidth]{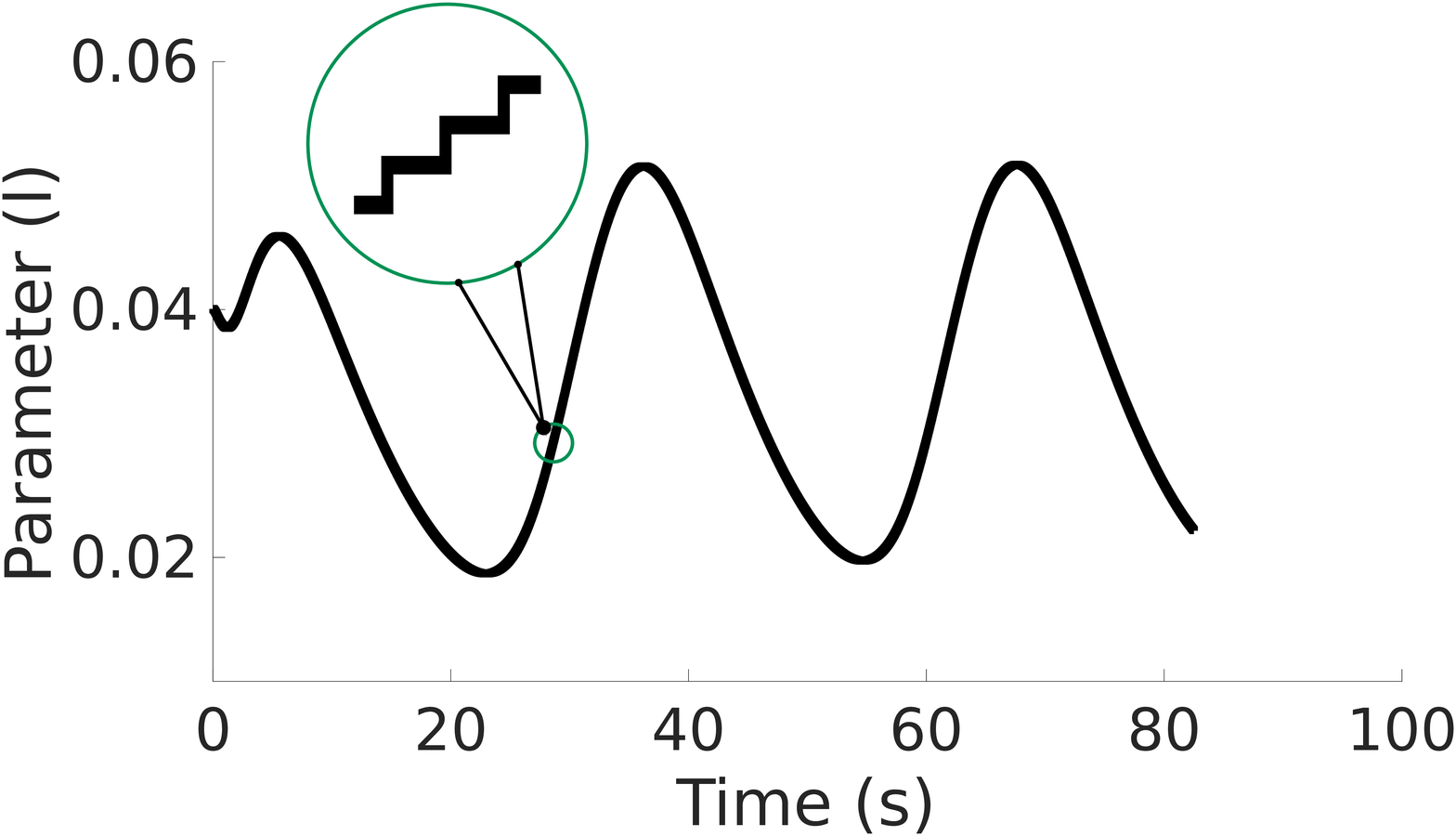}%
    ~\includegraphics[width=0.235\textwidth]{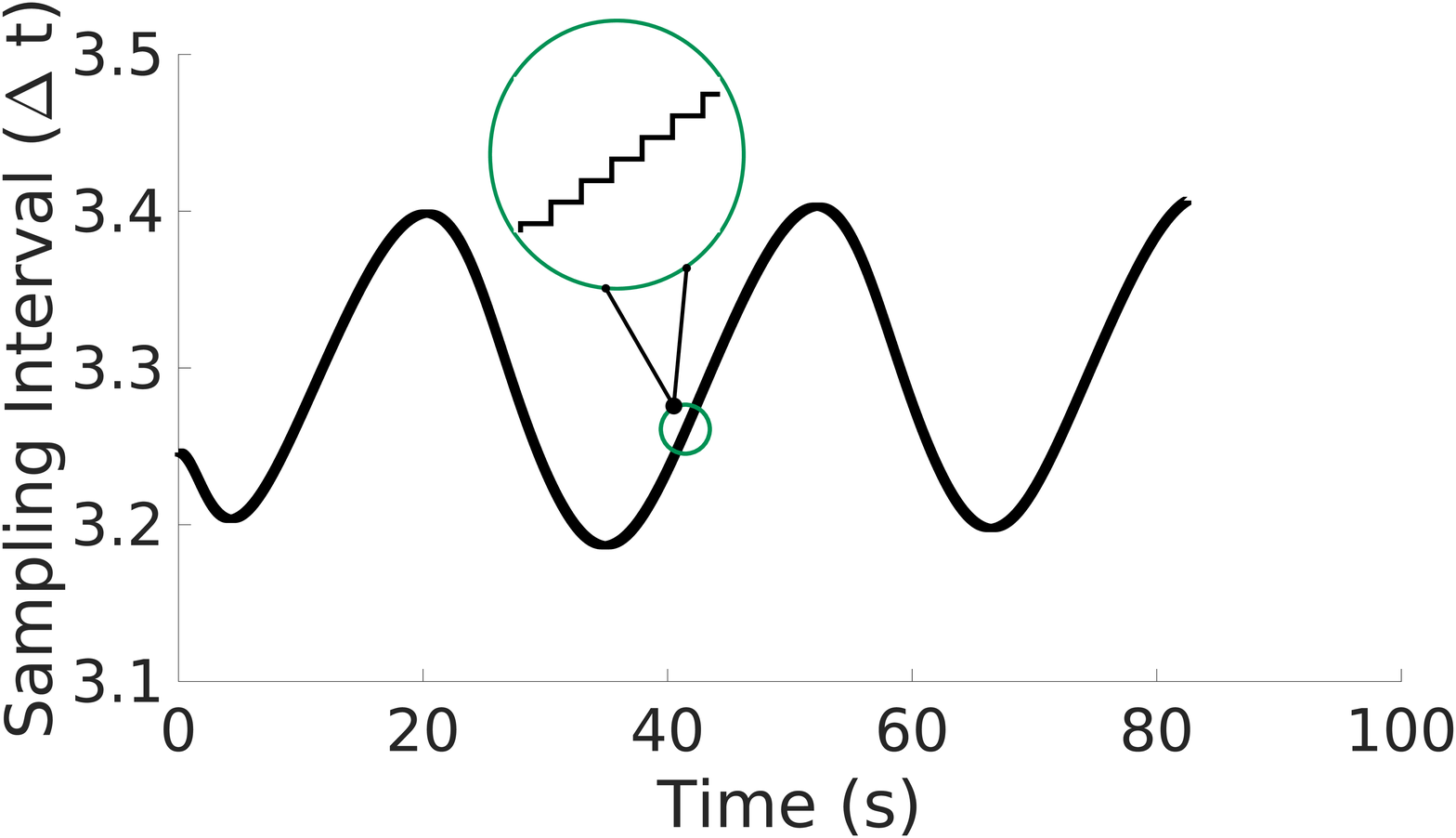}
    \caption{Parameter (left) and sampling horizon (right) from PMPC simulation, which oscillate because of the ellipsoidal reference trajectory in \eqref{eq:reference}.  Due to the sharp maneuvers required, the time horizon shortens and the parameter increases on the left and right sides of the ellipse.  On flatter regions, the PMPC reduces the parameter and increases the sampling time.  The zoomed portion displays the discrete nature of the PMPC solution.}
    \label{fig:parameter_interval_sim}
\end{figure}

\begin{figure}[tbp]
    \centering
    \includegraphics[width=0.48\textwidth]{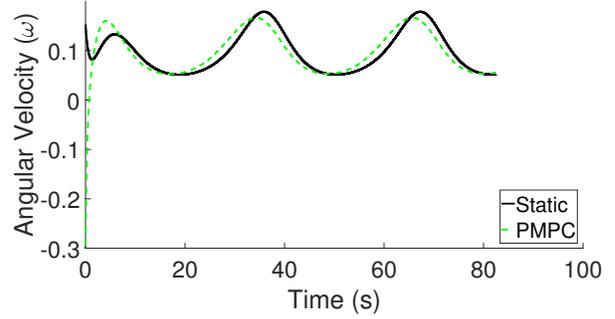}
    \caption{Angular velocity ($\omega$) during the simulation.  The simulation shows that the static selection (solid line) and PMPC method (dashed line) both generate similar angular velocity values.}
    \label{fig:omega_sim}
\end{figure}

This section proposes cost functions based on the results in Sec.~\ref{sec:main-results} and expresses the PMPC problem to be solved in simulation and on the Robotarium.  Furthermore, this section also statically parameterizes the NID to provide a baseline comparison to the PMPC strategy.  In this case, the particular setup involves a mobile robot tracking an ellipsoidal reference signal

\begin{figure*}[tp]
    \centering
    \begin{subfigure}[t]{0.32\textwidth}
        \includegraphics[width=\textwidth]{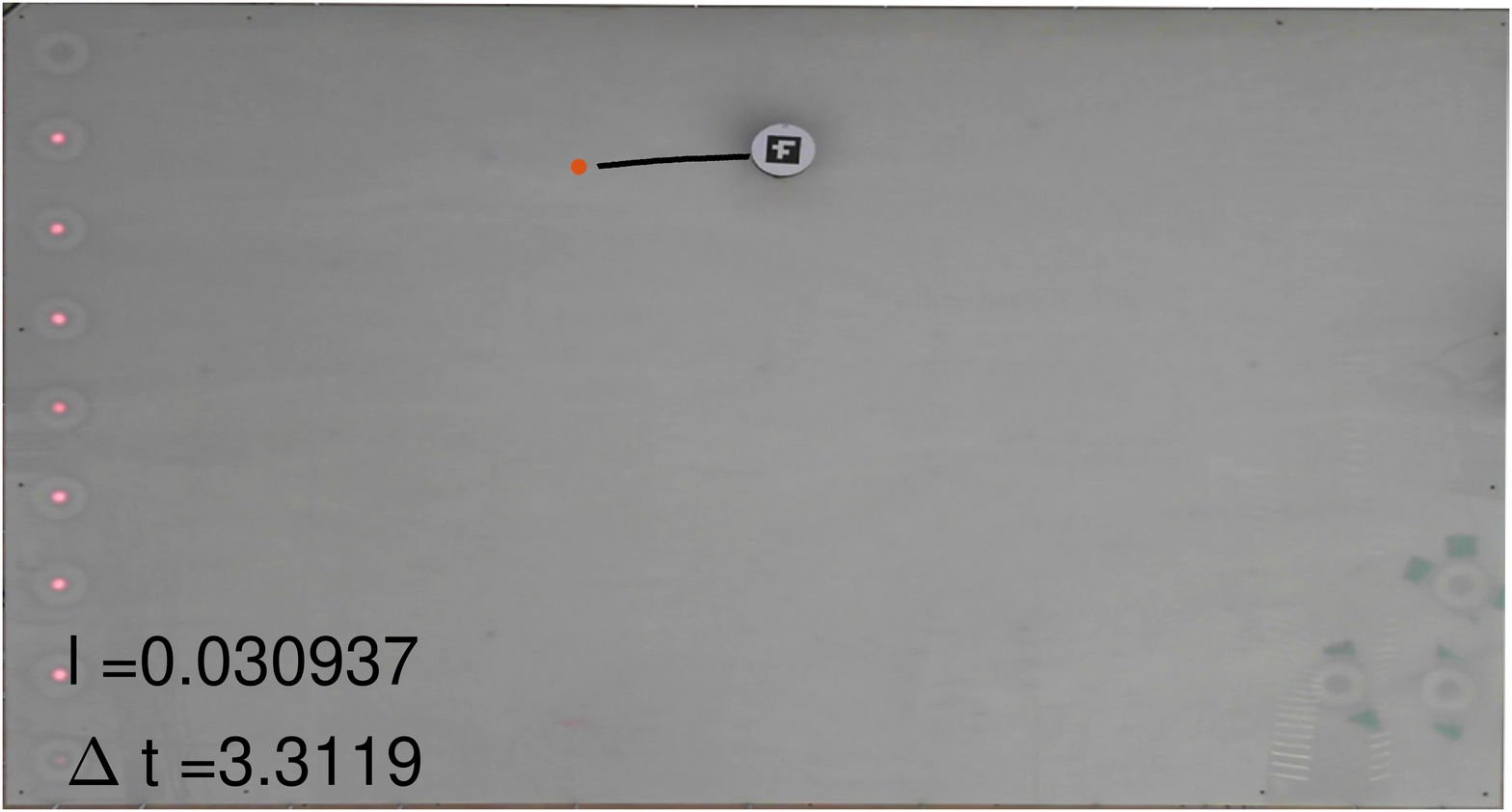}
    \end{subfigure}
    \begin{subfigure}[t]{0.32\textwidth}
        \includegraphics[width=\textwidth]{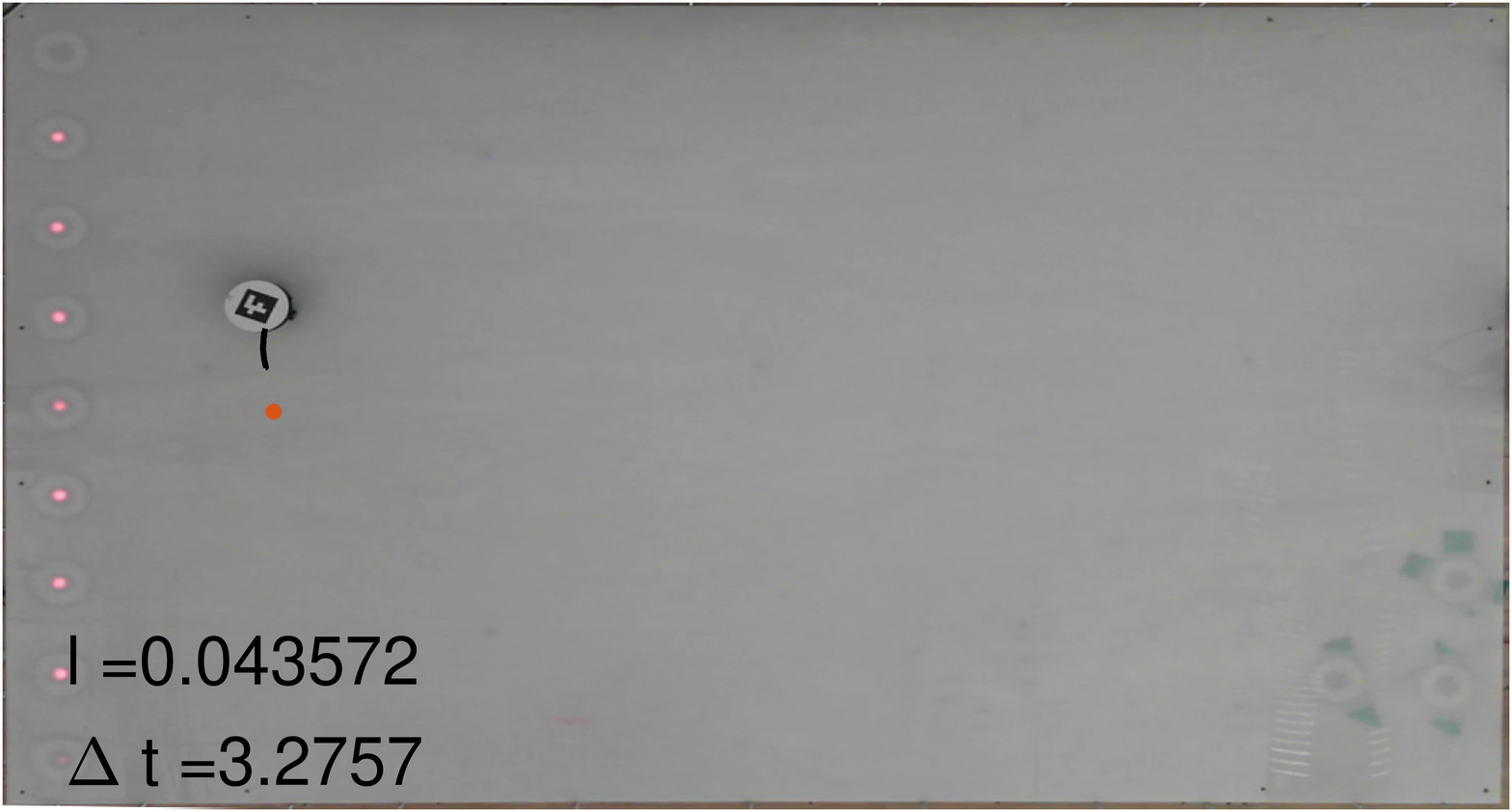}
    \end{subfigure}
    \begin{subfigure}[t]{0.32\textwidth}
        \includegraphics[width=\textwidth]{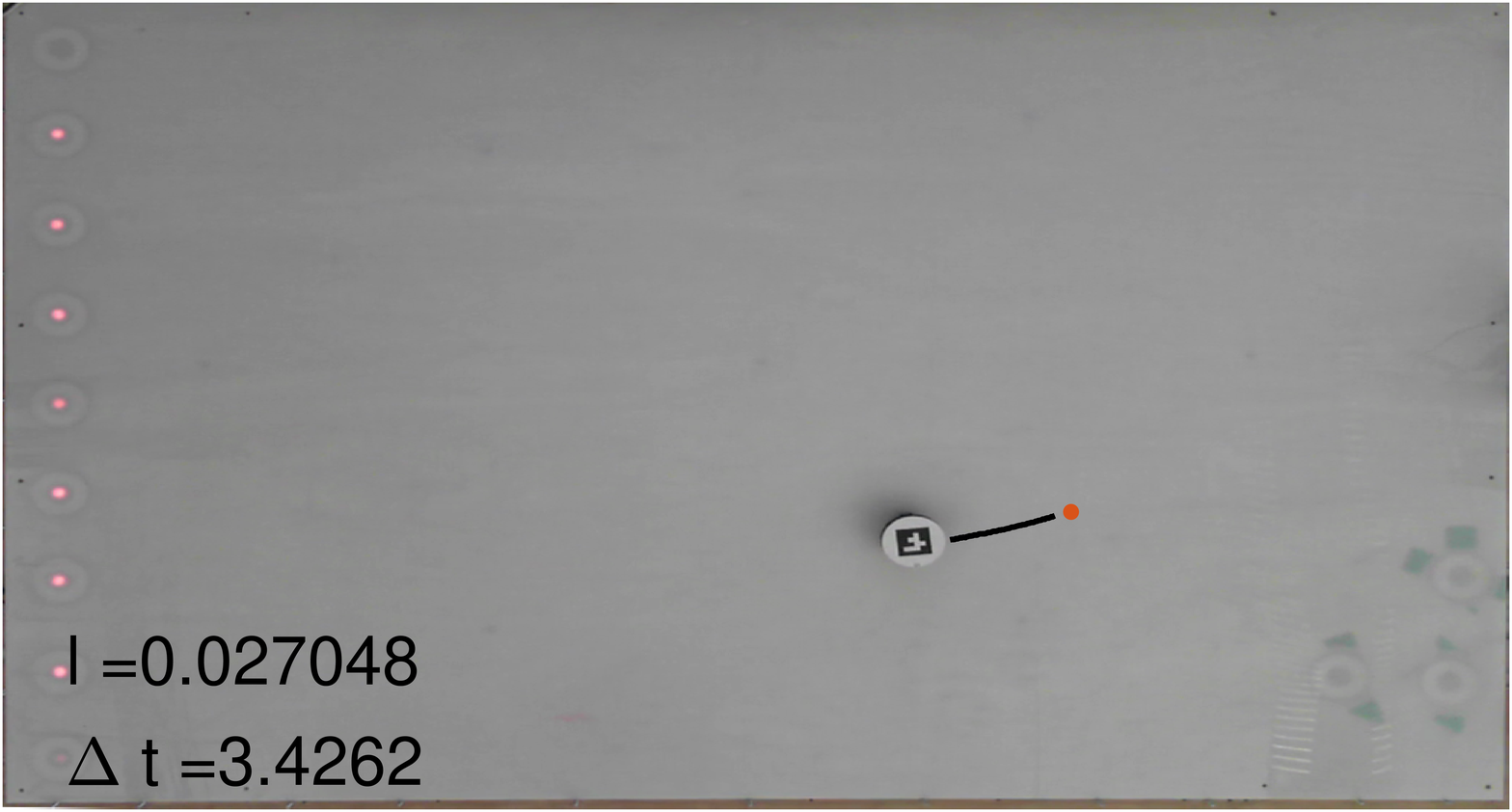}
    \end{subfigure}
    \caption{Robot during the PMPC experiment.  This figure shows that the parameter grows and sampling horizon shrinks when the robot must perform more complex maneuvers (i.e., on the left and right sides of the ellipse).  Over the flatter portions of the ellipse, the parameter increases and sampling horizon (solid line) reduces, allowing the robot to track the reference (solid circle) more closely.}
    \label{fig:picture-set-one}
\end{figure*}

\begin{equation}
    \label{eq:reference}
    r(t) = 
    \begin{bmatrix}
        0.4\cos((1/10)t) \\ 
        0.2\sin((1/10)t) 
    \end{bmatrix} .
\end{equation}

For a single-integrator system, the controller 
\begin{equation}
    u_{si} = x_{si} - r + \dot{r} 
\end{equation}
drives the single-integrator system to the reference exponentially quickly.  Utilizing the transformation in Sec.~\ref{sec:main-results} yields the controller 
\begin{align}
    u_{x} &= R_{l}(\theta)^{-1}(r - x_{si} + \dot{r}) \\
    &= R_{l}(\theta)^{-1}(r - (\bar{x} + l R(\theta)e_{1}) + \dot{r}) .
\end{align}
The GRITSbots of the Robotarium (shown in Fig.~\ref{fig:gritsbot}) have a wheel radius and base length of 
\begin{equation}
    r_{w} = 0.005~m,~ l_{w} = 0.03~m.
\end{equation}
Furthermore, their maximum forward velocity is 
\begin{equation}
    \bar{v} = 0.1~m/s .
\end{equation}
For this problem, we also consider the sampling cost
\begin{equation}
    C(\Delta t) = \dfrac{1}{\Delta t},
\end{equation}
which prevents the time horizon from becoming too small (i.e., the cost penalizes small time horizons). 

Substituting these values into \eqref{eq:pmpc_cost}, the particular PMPC problem to be solved is 
\begin{align}
    & \argmin_{l \in \R, \Delta t \in \R_{\geq 0}} \sint{t_{0}}{t_{0} \+ \Delta t}  (\beta - 1)((l_{w}/r_{w})e_{2}R_{l}(\theta)^{-1}(r - x_{si} \+ \dot{r}))^{2} \\
    & \qquad\qquad\quad + \beta l^{2} ds + (1/\Delta t) \\
    & \text{s.t. } \dot{x} = 
    \begin{bmatrix}
        R(\theta)e_{1} & 0 \\
        \mathbf{0} & 1
    \end{bmatrix} 
    R_{l}(\theta)^{-1}(r - x_{si} + \dot{r}) \\
    & \quad~ x(t_{0}) = x_{t_{0}} ,
\end{align}
where $l$ and $\Delta t$ are the decision variables and $\beta = 0.01$.  Both simulation and experimental results utilize Alg.~\ref{alg:mpc} to solve for the optimal parameters and time horizon online with the step-size values 
\begin{equation}
    \gamma_{1} = 0.001,~ \gamma_{2} = 0.01 .
\end{equation}

Each experiment initially executes Alg.~\ref{alg:mpc} to termination; then, steps are performed each iteration to ensure that the current values stays close to the locally optimal solution realized by Alg.~\ref{alg:mpc}.  In particular, each iteration takes $0.033~s$, which is the Robotarium's sampling interval. 

\begin{figure}[tp]
    \centering
    \includegraphics[width=0.48\textwidth]{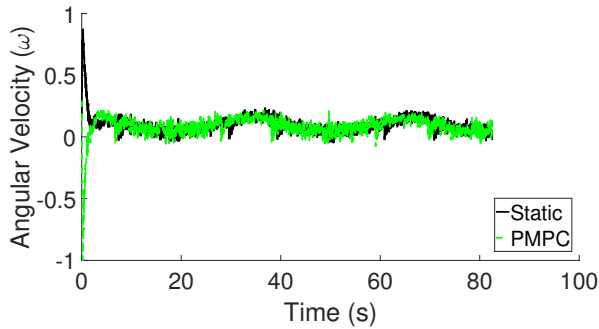}
    \caption{Angular velocity of robots for PMPC method (dashed line) versus static parameterization (solid line).  In this case, both methods generate similar angular velocities, but the PMPC method produces better tracking.}
    \label{fig:omega_exp}
\end{figure}

\begin{figure}[bp]
    \centering
    \includegraphics[width=0.235\textwidth]{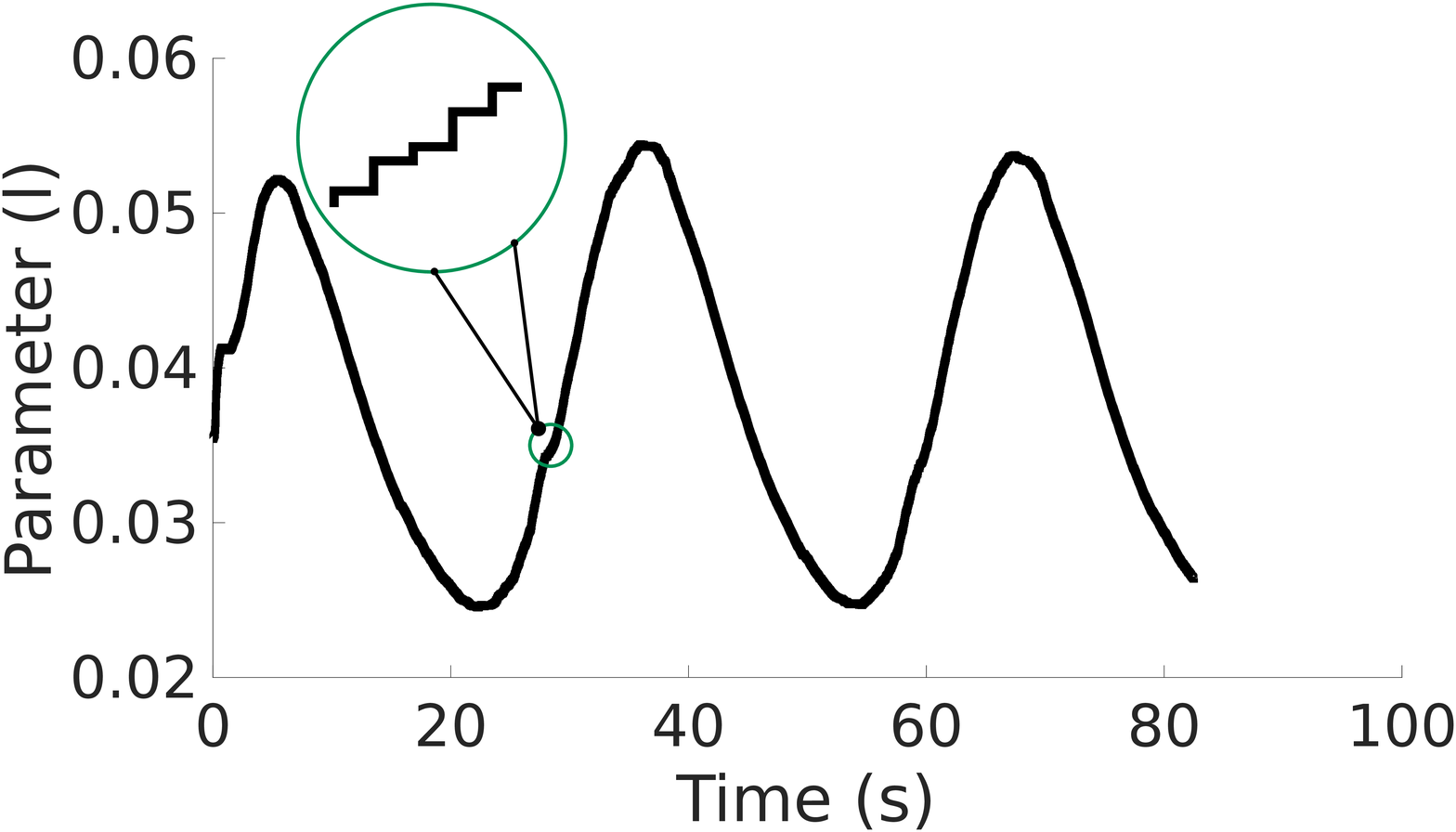}%
    ~\includegraphics[width=0.235\textwidth]{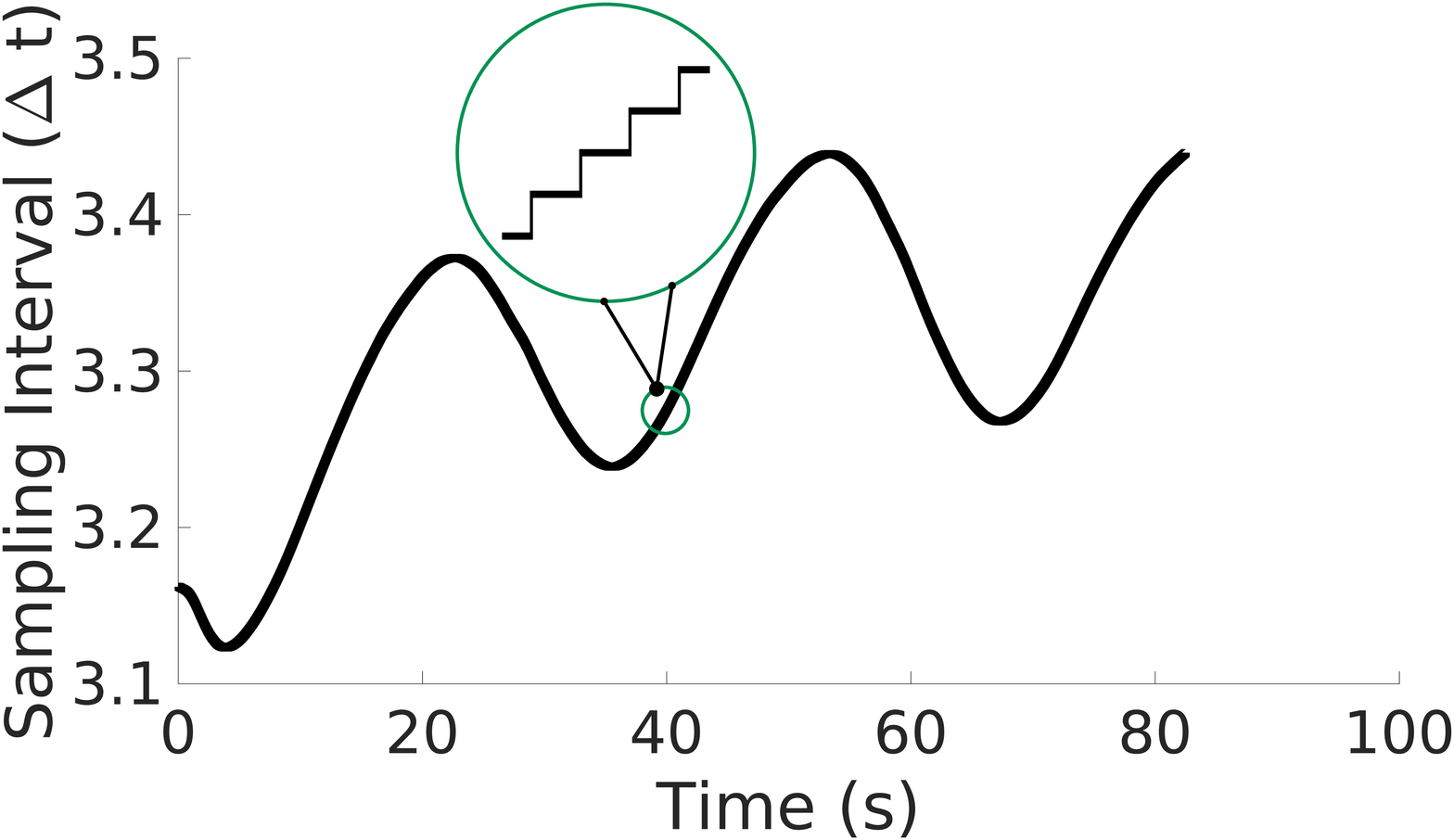}
    \caption{Parameter (left) and sampling horizon (right) from PMPC experiment on the Robotarium.  The PMPC program reduces the sampling horizon and increases the parameter to cope with the sharp maneuvers required at the left and right sides of the ellipse.  On flatter regions, the PMPC decreases the parameter and increases the time horizon, providing better reference tracking.  The zoomed portion illustrates the discrete nature of the PMPC solution.}
    \label{fig:parameter_interval_exp}
\end{figure}

For comparison, the one-time selection method stems directly from the abstraction cost formulated in Sec.~\ref{subsec:an-optimal-one-time-selection} with $\alpha = 0.99$.  This assignment to $\alpha$ in \eqref{eq:optimal_static_l} implies that 
\begin{equation}
    l^{*} = 0.078 .
\end{equation}
Note that this value of $l^{*}$ is only for the one-time selection.  The PMPC method induces different parameter values every $0.033~s$.

\subsection{Simulation Results}
\label{subsec:simulation-results}

This section contains the simulation results for the method described in Sec.~\ref{subsec:experiment_setup}.  In particular, the simulation compares the proposed PMPC method to the one-time selection process in Sec.~\ref{subsec:an-optimal-one-time-selection}, showing that the PMPC method can outperform the one-time selection.  Fig.~\ref{fig:omega_sim} shows the simulated angular velocities, and Fig.~\ref{fig:parameter_interval_sim} shows the parameter and sampling horizon evolution.  Both methods generate similar control inputs.  However, Fig.~\ref{fig:parameter_interval_sim} demonstrates that the PMPC method selects smaller parameter values, implying that this method provides better reference tracking.  

Additionally, Fig.~\ref{fig:parameter_interval_sim} also shows that the sampling horizon shortens and the parameter increases around the left and right portions of the ellipse, because these regions require sharper maneuvers and incur a higher maneuverability cost.  Furthermore, the ellipsoidal reference trajectory induces the oscillations in Fig.~\ref{fig:parameter_interval_sim}.  Overall, these simulated results show that the PMPC method can outperform a static parameterization.

\subsection{Experimental Comparison}
\label{subsec:experimental-comparison}

This section contains the experimental results of the implementation described in Sec.~\ref{subsec:experiment_setup}.  The physical experiments for this paper were deployed on the Robotarium and serve to highlight the efficacy and validity of applying the PMPC approach on a real system.  Additionally, the experiments display the propriety of the maneuverability cost outlined in Sec.~\ref{subsec:man-cost}.  

Figs.~\ref{fig:omega_exp}-\ref{fig:parameter_interval_exp} display the angular velocity of the mobile robot, the sampling horizon, and the parameter selection, respectively.  As in the simulated results, Fig.~\ref{fig:omega_exp} shows that the static parameterization and PMPC method produce similar angular velocities, and Fig.~\ref{fig:parameter_interval_exp} shows that the PMPC method is able to adaptively adjust the parameter and sampling horizon to handle variations in the reference signal.

Moreover, on a physical system, the PMPC method still adjusts the time horizon and parameter to account for maneuverability requirements.  For example, on the left and right sides of the ellipse, the maneuverability cost rises, because the reference turns sharply.  Thus, the parameter increases and the sampling horizon decreases.  Over flat portions of the ellipse, the maneuverability cost decreases, permitting the extension of the time horizon and reduction of the parameter (i.e., better tracking).  That is, reductions of the maneuverability cost permit decreasing the parameter $l$, allowing the PMPC strategy to outperform the static parameterization.  Furthermore, the decrease of the sampling horizon during high-maneuverability regions accelerates the execution of Alg.~\ref{alg:mpc}, which is useful in a practical sense.

\section{Conclusion}
\label{sec:conclusion}

This work presented a variable-sampling-horizon Parametric Model Predictive Control (PMPC) method that allows for optimal parameter and sampling horizon selection with the application of controlling differential-drive mobile robots.  To formulate an appropriate cost for the PMPC strategy, this article discussed a class of Near-Identity Diffeomorphisms (NIDs) that allow the transformation of single-integrator algorithms to unicycle-modeled systems.  Additionally, this work showed an inherent trade-off induced by the NID and formulated precision and maneuverability costs that allow for the optimal parameterization of the NID via a PMPC program.  Furthermore, simulation and experimental results were produced that illustrated the validity of the proposed costs and the efficacy of the PMPC method.

\bibliographystyle{ieeetr}
\bibliography{ref.bib}

\end{document}